%% file: main.tex
\def\nips{0} %set 1 for nips version, and 0 for arxiv version

\pdfoutput=1

\ifnum\nips=1

\documentclass{article}

\usepackage[final]{neurips_2019}

\usepackage[utf8]{inputenc} % allow utf-8 
\usepackage[T1]{fontenc}    % use 8-bit T1 

\usepackage{booktabs}       % 
\usepackage{amsfonts}       % blackboard math 
\usepackage{nicefrac}       % compact symbols 
\usepackage{microtype}      % microtypography

\else
\documentclass[a4paper]{article}
\usepackage[a4paper,top=3.5cm,bottom=3cm,left=3.5cm,right=3.5cm,marginparwidth=1.5cm]{geometry}
\fi

\usepackage{hyperref}       % hyperlinks
\usepackage{url}            % simple URL 

\usepackage{amsthm}
\usepackage{bbm}
\usepackage{times}
\usepackage{graphicx,color}
\usepackage{array,float}

\usepackage{xcolor}
\usepackage{amstext,amssymb,amsmath}
\usepackage{hyphenat}
\usepackage{verbatim}
\usepackage{bm}
\usepackage{paralist}
\usepackage{ulem}
\usepackage{todonotes}
\usepackage{paralist}\usepackage{bbm}
\usepackage{wrapfig}

\usepackage[noend]{algorithmic}
\usepackage{algorithm}
\usepackage{xparse,etoolbox}
\usepackage{threeparttable}
\usepackage{bbm}
\usepackage{enumitem}
\usepackage{dsfont}

\newtheorem{lem}{Lemma}[section]
\newtheorem{theorem}{Theorem}
\newtheorem{remark}[theorem]{Remark}

\newtheorem{thm}[lem]{Theorem}

\newtheorem{defn}[lem]{Definition}

\newtheorem{claim}[lem]{Claim}

\renewcommand{\paragraph}[1]{\vspace{3pt}\noindent\textbf{#1}}

\newcommand{\cX}{\ensuremath{\mathcal{X}}}

\newcommand{\cZ}{\ensuremath{\mathcal{Z}}}

\newcommand{\tH}{\widetilde{\mathcal{H}}}
\newcommand{\tT}{\widetilde{T}}

\newcommand{\hx}{\hat{x}}

\newcommand{\tlh}{\tilde{h}}

\newcommand{\tx}{\tilde{x}}

\newcommand{\tz}{\tilde{z}}
\newcommand{\hz}{\hat{z}}
\newcommand{\zz}{z^0}

\newcommand{\zpub}{z^{\sf pub}}
\newcommand{\zprv}{z^{\sf prv}}

\newcommand{\negl}{{\sf negl}}

\newcommand{\Aprv}{\mathcal{A}_{\sf{SSPP}}}

\newcommand{\hpv}{h_{\mathsf{priv}}}

\newcommand{\bad}{\sf Bad}

\newcommand{\pr}[2]{\underset{#1}{\mathbb{P}}\left[ #2 \right]}
\newcommand{\ex}[2]{\underset{#1}{\mathbb{E}}\left[ #2 \right]}

\newcommand{\zeroB}{\ensuremath{\mathbf{0}}}

\newcommand{\eps}{\epsilon}
\newcommand{\Ber}{{\sf Ber}}

\newcommand{\cA}{\mathcal{A}}

\newcommand{\cO}{\mathcal{O}}
\newcommand{\Sprv}{S_{\sf priv}}
\newcommand{\Spub}{S_{\sf pub}}
\newcommand{\Tpub}{T_{\sf pub}}
\newcommand{\hSpub}{\widehat{S}_{\sf pub}}
\newcommand{\npub}{n_{\sf pub}}
\newcommand{\nprv}{n_{\sf priv}}
\newcommand{\prvs}{{\sf PrivSamp}}
\newcommand{\pubs}{{\sf PubSamp}}

\newcommand{\cD}{\mathcal{D}}

\newcommand{\cG}{\mathcal{G}}
\newcommand{\cH}{\mathcal{H}}
\newcommand{\cHd}{\mathcal{H}_{\Delta}}

\newcommand{\hatm}{\hat{m}}

\newcommand{\Z}{\mathcal{Z}}

\newcommand{\cDp}{\mathcal{D}_{(p)}}

\newcommand{\ind}{{\mathbf{1}}}

\newcommand{\re}{\mathbb{R}}

\newcommand{\cB}{\mathcal{B}}

\newcommand{\tS}{\tilde{S}}
\newcommand{\tn}{\tilde{n}}

\newcommand{\ignore}[1]{}

\newcommand{\cR}{\mathcal{R}}

\newcommand{\bc}{\mathbf{c}}

\newcommand{\err}{\mathsf{err}}
\newcommand{\herr}{\widehat{\mathsf{err}}}
\newcommand{\VC}{\mathsf{VC}}
\newcommand{\dis}{\mathsf{dis}}
\newcommand{\hdis}{\widehat{\mathsf{dis}}}

\newcommand{\blocktheorem}[1]{%
	\csletcs{old#1}{#1}% Store \begin
	\csletcs{endold#1}{end#1}% Store \end
	\RenewDocumentEnvironment{#1}{o}
	{\par\addvspace{1.5ex}
		\noindent\begin{minipage}{\textwidth}
			\IfNoValueTF{##1}
			{\csuse{old#1}}
			{\csuse{old#1}[##1]}}
		{\csuse{endold#1}
		\end{minipage}
		\par\addvspace{1.5ex}}
}

\blocktheorem{theo}

\title{Limits of Private Learning with Access to Public Data}
\ifnum\nips=1
\author{
Noga Alon \\
Department of Mathematics\\ 
Princeton University \\
\texttt{nalon@math.princeton.edu} 
\And Raef Bassily \\ 
Department of Computer Science \& Engineering\\ 
The Ohio State University\\ \texttt{bassily.1@osu.edu} 
\And  Shay Moran \\
Google AI\\ 
Princeton\\  
\texttt{shaymoran1@gmail.com}}

\else 
\author{Noga Alon\thanks{Department of Mathematics, Princeton University. \texttt{nalon@math.princeton.edu}. Research supported in part by NSF grant DMS-1855464,
BSF grant 2018267, and the Simons Foundation.}
	\and Raef Bassily\thanks{Department of Computer Science \& Engineering, The Ohio State University. \texttt{bassily.1@osu.edu}. Part of this work was done while the author was visiting Simons Institute for the Theory of Computing. Research supported by NSF award AF-1908281, a Google Faculty Research Award, and OSU faculty start-up support.} \and Shay Moran\thanks{Google AI, Princeton. \texttt{shaymoran1@gmail.com}.}
}

\date{}
\fi
\begin{document}

\maketitle

\begin{abstract}
We consider learning problems where the training set consists of two types of examples: private and public. The goal is to design a learning algorithm that satisfies differential privacy only with respect to the private examples. This setting interpolates between private learning (where all examples are private) and classical learning (where all examples are public). 

We study the limits of learning in this setting in terms of private and public sample complexities. We show that any hypothesis class of VC-dimension $d$ can be agnostically learned up to an excess error of $\alpha$ using only (roughly) $d/\alpha$ public examples and $d/\alpha^2$ private labeled examples. This result holds even when the public examples are unlabeled. This gives a quadratic improvement over the standard~$d/\alpha^2$ upper bound on the public sample complexity (where private examples can be ignored altogether if the public examples are labeled). Furthermore, we give a nearly matching lower bound, which we prove via a generic reduction from this setting to the one of private learning without public data. 
  
\end{abstract}

\ifnum\nips=0

\input{intro.tex}

\input{prelim.tex}

\input{upper.tex}

\input{lower.tex}

\newpage
\bibliographystyle{alpha}
\bibliography{references}

\else

\input{NIPS/intro.tex}
\input{NIPS/prelim.tex}
\input{NIPS/upper.tex}
\input{NIPS/lower.tex}
\subsection*{Acknowledgements}
N. Alon's research is supported in part by NSF grant DMS-1855464,
BSF grant 2018267, and the Simons Foundation. R. Bassily's research is supported by NSF award AF-1908281, a Google Faculty Research Award, and OSU faculty start-up support. Part of this work was done while R. Bassily was visiting the Simons Institute for the Theory of Computing. 

\newpage
\bibliographystyle{alpha}
\bibliography{references}
\fi

\end{document}

%% file: intro.tex
\section{Introduction}

%\textcolor{red}{Shay:
%Main result (i) (upper bound): Let $\cH$ be a learnable class (i.e.\ $\mathsf{VC}(\cH) < \infty$).
%Then, $\cH$ is learnable by a semi-private algorithm whose public sample complexity is\ldots Moreover, the public input sample can be unlabeled.
%Main result (ii) (lower bound): Let $\cH$ be a class that is not DP-learnable. Then, every semi-private learner for $\cH$ requires at least\ldots
%Corollary (dichotomy): Every class satisfies precisely one of the following items: (i) it can be semi-privately learned without any public examples (i.e.\ it can be privately learned)
%(ii) Any semi-private learner for $\cH$ requires at least\ldots
%}

%\textcolor{red}{Shay:
%As for the more quantitative formulation of the lower bound: I suggest ignoring expressions such as $1/(20n^2\log n)$ or similar,
%and simply write that if $\cH$ \underline{can not} be learned privately with privacy parameters $(\epsilon,\delta)$ and utility parameters $(\alpha,\beta=1/12)$
%then it \underline{can not} be semi-privately learned with privacy parameters $(\epsilon,\delta)$ and utility parameters\ldots
%Then, we can exemplify its implication for thresholds (in a corollary or discussion).} 

In this work, we study a relaxed notion of differentially private (DP) supervised learning 
which was introduced by Beimel et al.\ in~\cite{beimel2013private}, where it was coined \emph{semi-private learning}.
In this setting, the learning algorithm takes as input a training set that is comprised of two parts: (i) a private sample that contains personal and sensitive information, and (ii) a ``public'' sample  that poses no privacy concerns. 
We assume that the private sample is always labeled, while the public sample can be either labeled or unlabeled. The algorithm is required to satisfy DP only with respect to the private sample. The goal is to design algorithms that can exploit as little public data as possible to achieve non-trivial gains in accuracy (or, equivalently savings in sample complexity) over standard DP learning algorithms, while still providing strong privacy guarantees for the private dataset. Similar settings have been studied before in literature (see ``Related Work'' section below). 

There are several motivations for studying this problem. First, in practical scenarios, it is often not hard to collect reasonable amount of public data from users or organizations. For example, in the language of consumer privacy, there is considerable amount of data collected from the so-called ``opt-in'' users, who voluntarily offer or sell their data to companies or organizations. Such data is deemed by its original owner to have no threat to personal privacy. There are also a variety of other sources of public data that can be harnessed. Moreover, in many scenarios, it is often much easier to collect unlabeled than labeled data. 
%As we describe in the sequel, our positive result gives an algorithm that is able to learn even when the public data is unlabeled,
%and our negative results establish limitations that apply even in the setting when the public data is labeled. 

Another motivation emerges from several pessimistic results in DP learning that either limit or eliminate the possibility of differentially private learning, even for elementary problems such as one-dimensional thresholds
which are trivially learnable without privacy constraints \cite{bun2015differentially,almm19}. It is therefore natural to explore whether a small amount of public data circumvents these impossibility results.

%Another motivation is the fact that there are several pessimistic results in the standard setting of DP learning (when all examples are private) that either limit or eliminate the possibility of producing accurate models (classifiers) in several problems, including problems that are ``easy'' to learn without privacy. For example, \cite{almm19, bun2015differentially} show that learning the class of threshold functions over the reals is impossible under DP. Without privacy, learning this class up to error $\alpha$ requires only a training sample of size $1/\alpha$ examples.  %Moreover, unlike the classical (non-private) setting where we have a notion like the VC-dimension, there is still no \emph{universal} characterization of the sample complexity of approximate differentially private learners unless the concept class is finite.

A third motivation arises from the following observation: consider a learning problem in which the marginal distribution $\cD_\cX$ over the domain $\cX$ is completely known to the algorithm, but the target concept $c:\cX\to\{0,1\}$ is unknown. One can show that in this setting every VC class can be learned privately with (roughly) the same sample complexity as in the standard, non-private, case. The other extreme is the standard PAC-setting in which both $\cD_\cX$ and $c$ are unknown to the algorithm. As mentioned earlier, in this case even very simple classes such as one-dimensional thresholds can not be learned privately. In the setting considered in this work, the distribution $\cD_\cX$ is unknown but the learner has access to some public examples from it. This naturally interpolates between these two extremes: the case when $\cD_\cX$ is unknown that corresponds to having no public examples, and the case when $\cD_\cX$ is known that corresponds to having an unbounded amount of public examples. It is therefore natural to study the intermediate behaviour as the number of public examples grows from 0 to $\infty$. 
The same question can be also asked in the ``easier'' case where the public examples are labeled. 

We will generally refer to the setting described above as \emph{semi-private learning}, and to algorithms in that setting as \emph{semi-private learners}. (See Section~\ref{sec:prelim}, for precise definitions.) 
%We will also be more specific in our terminology to distinguish between the cases with labeled and unlabeled public examples.) 
Following previous works in private learning, we consider two types of semi-private learners: those that satisfy the notion of \emph{pure} DP (the stronger notion of DP), as well as those that satisfy \emph{approximate} DP. We will call the former type \emph{pure} semi-private learners, and call the latter \emph{approximate} semi-private learners.

\subsection*{Main Results} 
In this work we concentrate on the sample complexity of semi-private learners in the agnostic setting.
We especially focus on the minimal number of public examples with which it is possible to learn every VC class.

\begin{enumerate}[leftmargin=*] 
    \item \textbf{Upper bound:} Every hypothesis class $\cH$ can be learned up to excess error $\alpha$ by a \emph{pure semi-private} algorithm whose \emph{private} sample complexity is (roughly) $\VC(\cH)/\alpha^2$ and \emph{public} sample complexity is (roughly) $\VC(\cH)/\alpha$. Moreover, the input public sample can be \emph{unlabeled}. 
    
    Recall that $\VC(\cH)/\alpha^2$ examples are necessary to learn in the agnostic setting (even without privacy constraints); therefore, this result establishes a quadratic saving.
    
    \item \textbf{Lower bound:} Assume $\cH$ has an infinite {\it Littlestone dimension}\footnote{The Littlestone dimension is a combinatorial parameter that arises in online learning~\cite{Littlestone87,Ben-DavidPS09}.}. Then, any \emph{approximate} semi-private learner for $\cH$ must have \emph{public} sample complexity~$\Omega(1/\alpha),$ where $\alpha$ is the excess error. This holds even when the public sample is \emph{labeled}. 
    
    One example of a class with an infinite Littlestone dimension is the class of thresholds over $\mathbb{R}$. 
    This class has VC dimension $1$, and therefore demonstrates that the upper and lower bounds above nearly match.
    \item \textbf{Dichotomy for pure semi-private learning:} Every hypothesis class~$\cH$  satisfies \emph{exactly} one of the following: 
    \begin{itemize}
        \item[(i)] $\cH$ is learnable by a \emph{pure} DP algorithm, and therefore can be semi-privately learned without any public examples. 
        \item[(ii)] Any \emph{pure} semi-private learner for $\cH$ must have public sample complexity $\Omega\left(1/\alpha\right)$,
        where $\alpha$ is the excess error.
    \end{itemize}
\end{enumerate}

\subsubsection*{Techniques} 
\paragraph{Upper bound:} The idea of the construction for the upper bound is to use the (unlabeled) public data to construct a finite class $\cH'$ that forms a ``good approximation'' of the original class $\cH$, then reduce the problem to DP learning of a finite class. Such approximation is captured via the notion of $\alpha$-covering (Definition~\ref{defn:cover}). By standard uniform-convergence arguments, it is not hard to see that (roughly) $\VC(\cH)/\alpha^2$ public examples suffice to construct such an approximation. We show that the number of public examples can be reduced to only
about $\VC(\cH)/\alpha$, even in the agnostic setting. 
Our construction is essentially the same as a construction due to Beimel et al.~\cite{beimel2013private}, but our proof technique is different 
(see the ``Related Work'' section for a more detailed comparison).
%Our construction is very similar to a construction due to Beimel et al.~\cite{beimel2013private} (see the ``Related Work'' section for a more detailed comparison).

\paragraph{Lower bounds:}  The lower bounds boil down to a {\it public-data-reduction lemma} which shows that
if we are given a semi-private learner whose public sample complexity is $<<1/\alpha$, 
we can transform it to a \emph{fully} private learner (which uses no public examples) whose excess error is a small constant (say ${1}/{100}$). Stated contra-positively, this implies that if a class can not be privately learned up to an excess loss of~$1/100$ then it can not be semi-privately learned with $<<1/\alpha$ public examples. This allows us to exploit known lower bounds for private learning to derive lower on the public sample complexity.

%that one can reduce the number of public examples at the expense of a proportional amplification in the error. In particular, we show that if a hypothesis class can be \emph{semi-privately} learned up to an excess error of $\alpha$ with less than $\frac{1}{1000\alpha}$ public examples then it can also be \emph{privately} learned without any public examples and excess error of at most $<\frac{1}{10}$.

\subsubsection*{Related Work}

% \textcolor{red}{Shay: I think we should be focused and only consider works which consider such public/private settings. (in particular not to survey general works on private learning.) I added a paragraph for Beimel et al. Which other papers are there? (There is your paper from last NIPS, and?)}

%{\color{blue}New related work (read)}

%The most related work to ours is the work of Beimel et al.\,\cite{beimel2013private}, which focuses on the realizable case of semi-private learning and give an upper bound on the sample complexity. The algorithm we use in our upper bound is essentially the same as theirs. However, our analysis for the sample complexity differs from theirs: our argument relies on the notion of $\alpha$-coverings,  which provides a direct argument that extends to the agnostic case. 

Our algorithm for the upper bound is essentially the same as a construction due to Beimel et al.\,\cite{beimel2013private}. Although \cite{beimel2013private} focuses on the realizable case of semi-private learning, their analysis can be extended to the agnostic case to yield a similar upper bound to the one we present here. However, the proof technique we give here is different from theirs. In particular, our proof relies on and emphasizes the use of $\alpha$-coverings, 
which provides a direct argument for both the realizable and agnostic case. We believe the notion of $\alpha$-covering can be a useful tool in the analysis of other differentially private algorithms even outside the learning context.

\noindent There are also several other works that considered similar problems. A similar notion known as ``label-private learning'' was considered in \cite{chaudhuri2011sample} (see also references therein) and in \cite{beimel2013private}.
% \textcolor{red}{Shay: perhaps cite also Beimel et al.\ in this context? They have the strongest results in this contect, don't they?}
In this notion, only the labels in the training set are considered private. This notion is weaker than semi-private learning. In particular, any semi-private learner can be easily transformed into a label-private learner. Another line of work consider the problem of private knowledge transfer \cite{hamm2016learning}, \cite{papernot2017semi}, \cite{papernot2018scalable}, and \cite{bassily2018NIPS}. In this problem, first a DP classification algorithm with input private sample is used to provide labels for an unlabeled public dataset. Then, the resulting dataset is used to train a non-private learner. The work of \cite{bassily2018NIPS} gives upper bounds on private and public sample complexities in the setting when the DP algorithm is required to label the public data in an online fashion. Their bounds are thus not comparable to ours.

%% file: prelim.tex
\section{Preliminaries}\label{sec:prelim}
\subsection*{Notation} 
For $n\in\mathbb{N}$, we use $[n]$ to denote the set $\{1,\ldots,n\}$.
We use standard asymptotic notation $O,\Omega,o,\omega$.
A function $f:\mathbb{N}\to[0,1]$ is said to be negligible if $f(n) = o(n^{-d})$ for every $d\in\mathbb{N}$.
The statement ``$f$ is negligible'' is denoted by $f=\negl(n)$. 

\noindent We use standard notation from the supervised learning literature (see, e.g.~\cite{shalev2014understanding}).
Let $\cX$ denote an arbitrary domain, 
let $\mathcal{Z} =\cX\times\{0,1\} $ denote the examples domain, and let $\mathcal{Z}^* = \cup_{n=1}^{\infty}\mathcal{Z}^n$.
A function $h:\cX\to\{0,1\}$ is called a concept/hypothesis,
a set of hypotheses $\cH\subseteq \{0,1\}^{\cX}$ is called a concept/hypothesis class.
The VC dimension of $\cH$ is denoted by $\VC(\cH)$. We use $\cD$ to denote a  distribution over $\mathcal{Z}$, and $\cD_{\cX}$ to denote the marginal distribution over $\cX$. 
We use $S\sim \cD^n$ to denote a sample/dataset $S=\left\{(x_1, y_1), \ldots, (x_n, y_n)\right\}$ of $n$ i.i.d. draws from $\cD$. 
%In the realizable setting, the data distribution $\cD$ is described by a distribution $\cD_{\cX}$ over $\cX$ and a hypothesis $h^*$ such that $(x, y)\sim \cD \iff x\sim \cD_{\cX}, y=h^*(x)$. In such case, we use $S\sim\left(\cD_{\cX}, h^*\right)^n$ to denote a sample of $n$ i.i.d. draws from $\cD$.

\paragraph{Expected error:} The expected/population error of a hypothesis $h:\cX\rightarrow \{0, 1\}$ with respect to a distribution $\cD$ over $\mathcal{Z}$ is defined by $\err(h; \cD)\triangleq \ex{(x, y)\sim\cD}{\ind\left(h(x)\neq y\right)}$.

\noindent A distribution $\cD$ is called {\it realizable by $\cH$} if there exists $h^*\in \cH$ such that $\err(h^*; \cD)=0$.
In this case, the data distribution $\cD$ is described by a distribution $\cD_{\cX}$ over $\cX$ and a hypothesis~$h^*\in\cH$. For realizable distributions, the expected error of a hypothesis $h$ will be denoted by $\err\left(h; ~\left(\cD_{\cX}, h^*\right)\right)\triangleq \ex{x\sim\cD_{\cX}}{\ind\left(h(x)\neq h^*(x)\right)}.$ 

\paragraph{Empirical error:} The empirical error of an hypothesis $h:\cX\rightarrow \{0, 1\}$ with respect to a labeled dataset $S=\left\{(x_1, y_1), \ldots, (x_n, y_n)\right\}$ will be denoted by $\herr\left(h; S\right)\triangleq \frac{1}{n}\sum_{i=1}^n \ind\left(h(x_i)\neq y_i\right).$

\paragraph{Expected disagreement:} The expected disagreement between a pair of hypotheses $h_1$ and $h_2$ with respect to a distribution $\cD_{\cX}$ over $\cX$ is defined as $\dis\left(h_1, h_2; ~\cD_{\cX}\right)\triangleq \ex{x\sim\cD_{\cX}}{\ind\left(h_1(x)\neq h_2(x)\right)}.$

\paragraph{Empirical disagreement:} The empirical disagreement between a pair of hypotheses $h_1$ and $h_2$ w.r.t. an unlabeled dataset $T=\left\{x_1, \ldots, x_n\right\}$ is defined as $\hdis\left(h_1, h_2; ~T\right)=\frac{1}{n}\sum_{i=1}^n\ind\left(h_1(x_i)\neq h_2(x_i)\right).$

\subsection*{Definitions}

\begin{defn}[Differential Privacy \cite{DMNS06, DKMMN06}]\label{defn:DP}
Let $\eps, \delta>0$. A (randomized) algorithm $\cA$ with input domain $\cZ^*$ and output range $\cR$ is called $(\eps,\delta)$-differentially private if for all pairs of datasets $S,S'\in \cZ^*$ that differs in exactly one data point, and every measurable $\cO \subseteq \cR$, we have
$$\Pr \left(\cA(S) \in \cO \right) \leq e^\eps \cdot \Pr \left(\cA(S') \in \cO \right) + \delta,$$
where the probability is over the random coins of $\cA$. When $\delta=0$, we say that $\cA$ is \emph{pure} $\eps$-differentially private.
\end{defn}

%In the sequel, we will usually use $\cX$ to denote the  $\cD_{\cX}$ to denote the distribution over the domain $\cX$ of the unlabeled data points.

We study learning algorithms that take as input two datasets: a private dataset $\Sprv$ and a public dataset $\Spub$,
and output a hypothesis $h:\cX\to\{0,1\}$. The public set entails no privacy constraint, whereas the algorithm is required to satisfy differential privacy with respect to $\Sprv$. The private set $\Sprv \in \left(\cX\times\{0, 1\}\right)^*$ is labeled. We distinguish between two settings of the learning problem depending on whether the public dataset is labeled or not. To avoid confusion, we will usually denote an \emph{unlabeled} public set as $\Tpub\in \cX^*$, and use $\Spub$ to denote a \emph{labeled} public set. We formally define learners in these two settings. 

% \textcolor{red}{Shay: perhaps add somewhere a comment that assuming $\eps=O(1)$ in the definition of private PAC learning is without loss of generality since there are standard methods to boost the dependence on $\eps$. (Actually, I vaguely recall that there may be some subtlety and that these boosting methods only apply when the initial $\eps\leq 1$.)}

\begin{defn}[$(\alpha, \beta, \eps, \delta)$- Semi-Private Learner]\label{defn:pp-learner}
Let $\cH\subset \{0, 1\}^{\cX}$ be a hypothesis class. 
A randomized algorithm $\cA$ is $(\alpha, \beta, \eps, \delta)$-SP (semi-private) learner for $\cH$ with private sample size $\nprv$ and public sample size $\npub$ if the following conditions hold: 

\begin{enumerate}
    \item For every distribution $\cD$ over $\mathcal{Z}=\cX\times \{0, 1\}$,
    given datasets $\Sprv\sim \cD^{\nprv}$ and $\Spub\sim \cD^{\npub}$ as inputs to $\cA$, with probability at least $1-\beta$ (over the choice of $\Sprv, ~\Spub,$ and the random coins of $\cA$), $\cA$ outputs a hypothesis $\cA\left(\Sprv, \Spub\right)=\hat{h}\in \{0, 1\}^{\cX}$ satisfying 
$$\err\left(\hat{h};~\cD\right)\leq \inf\limits_{h\in\cH}\err\left(h; ~\cD\right)+\alpha.$$  \label{cond:1}
    \item For all $S\in\mathcal{Z}^{\npub},$ $\cA\left(\cdot, S\right)$ is $(\eps, \delta)$-differentially private. \label{cond:2}
\end{enumerate}
When the second condition is satisfied with $\delta=0$ (i.e., pure differential privacy), we refer to $\cA$ as $(\alpha, \beta, \eps)$-SP learner (i.e., pure semi-private learner).
\end{defn}

As a special case of the above definition, we say that an algorithm $\cA$ is an $(\alpha, \beta, \eps, \delta)$-semi-privately learner for a class $\cH$ under the realizability assumption if it satisfies the first condition in the definition only with respect to all distributions that are realizable by $\cH.$

\begin{defn}[Semi-Privately Learnable Class]\label{defn:pp-class-learn}
We say that a class $\cH$ is semi-privately learnable if there are functions $\nprv:(0, 1)^2\rightarrow \mathbb{N}, ~\npub:(0, 1)^2\rightarrow \mathbb{N},$ {where $\npub(\alpha,\cdot) = o(1/\alpha^2)$}, and there is an algorithm $\cA$ such that for every $\alpha, \beta \in (0, 1),$ when $\cA$ is given private and public samples of sizes $\nprv=\nprv(\alpha, \beta),$ and $\npub=\npub(\alpha, \beta),$ it $\left(\alpha, \beta, 0.1, \negl\left(\nprv\right)\right)$-semi-privately learns $\cH$. 
%$\left(\alpha, \beta, \eps, \delta\right)$-SP learner for that class, with $\nprv<\infty$, $\npub=o\left(1/\alpha^2\right),~ \eps=O(1),$ and $\delta=$. 
%(where $\negl(n)$ is negligible function of $n$, i.e., $\negl(n)=o\left(\frac{1}{n^t}\right)~ \forall~ t>0.$) 
\end{defn}
Note that in the definition above, the privacy parameters are set as follows: $\eps=0.1$ and $\delta$ is negligible function in the private sample size (and $\delta=0$ for a \emph{pure} semi-private learner).

The choice of $\npub = o(1/\alpha^2)$ in the above definition is because taking $\Omega(\VC(\cH)/\alpha^2)$ public examples
suffices to learn the class without any private examples (see \cite{shalev2014understanding}). 
Thus, the above definition focuses on classes for which there is a non-trivial saving in the number of public examples
required for learning. Beimel et al.\, \cite{beimel2013private} were the first to propose the notion of \emph{semi-private} learners. Their notion is analogous to a special case of Definition~\ref{defn:pp-learner}, which we define next.

\begin{defn}[$(\alpha, \beta, \eps, \delta)$-Semi-Supervised Semi-Private Learner]\label{defn:ss-pp-learner}
The definition is analogous to Definition~\ref{defn:pp-learner} except that the public sample is \emph{unlabeled}. That is, $\cA$ is $(\alpha, \beta, \eps, \delta)$-SS-SP (semi-supervised semi-private) learner for a class $\cH$ with private sample size $\nprv$ and public sample size $\npub$ if the same conditions in Definition~\ref{defn:pp-learner} hold except that in condition~\ref{cond:1}, $\Spub\sim \cD^{\npub}$ is replaced with $\Tpub\sim\cD_{\cX}^{\npub}$, and  condition~\ref{cond:2} is replaced with 
``For all $T\in \cX^{\npub}$, $\cA\left(\cdot, T\right)$ is $(\eps, \delta)$-differentially private.''
\end{defn}

\noindent We define the notion of semi-supervised semi-privately learnable class $\cH$ in analogous manner as in Definition~\ref{defn:pp-class-learn}.

\paragraph{Private learning without public data:} In the standard setting of $(\eps, \delta)$-differentially private learning, %\cite{KLNRS08, beimel2013private, vadhan2017complexity}, 
the learner has no access to public data. We note that this setting can be viewed as a special case of Definitions~\ref{defn:pp-learner} and \ref{defn:ss-pp-learner} by taking $\npub=0$  (i.e., empty public dataset). In such case, we refer to the learner as $(\alpha, \beta, \eps, \delta)$-\emph{private learner}. As before, when $\delta=0,$ we call the learner \emph{pure private learner}. The notion of \emph{privately learnable} class $\cH$ is defined analogously to Definition~\ref{defn:pp-class-learn} with $\npub(\alpha, \beta)=0$ for all $\alpha, \beta$.  %In this case, we refer to such learner as \emph{$(\alpha, \beta, \eps, \delta)$-private learner.}
% \begin{defn}[Private Learning]\label{defn:priv-learner}
% A randomized algorithm $\cA$ is $(\alpha, \beta, \eps, \delta)$-private learner for a class $\cH$ with sample size $n$ if it satisfies Definition~\ref{defn:ss-pp-learner} with $\nprv=n$ and $\npub=0$. 

% More generally, we say that a hypothesis class $\cH$ is $(\alpha, \beta)$-privately learnable if there exists an $\left(\alpha,~ \beta, ~O(1),~ \negl(n)\right)$-private learner for that class with finite input sample size $n$. On the other hand, we say that a class $\cH$ is not privately learnable if $\cH$ is not $\left(\frac{1}{10}, \frac{1}{10}\right)$-privately learnable.
% \textcolor{red}{Shay: The latter addition is strange: usually, a class $\cH$ is not privately learnable if it does not satisfy the definition of private learnability. I am not sure what are you trying to convey here, perhaps you wish to emphasize that a class is not privately learnable if it fails to satisfy the definition for specific $(\alpha,\beta)$ (say $(1/10,1/10)$)? If this is the case then we should take it out of the definition environment (it is an observation rather than a definition).  }

% Analogously, we say that the class $\cH$ is $(\alpha, \beta)$-privately learnable under realizability if it is $(\alpha, \beta)$-privately learnable only with respect to all distributions realizable by $\cH.$
% \end{defn}
% As before, we that a class $\cH$ is $(\alpha, \beta, \eps, \delta)$-privately learnable if there is an $(\alpha, \beta, \eps, \delta)$-private learner for that class with finite $\nprv$. 

We will use the following lemma due to Beimel et al.~\cite{beimel2015learning}:

%, 
%which shows that every class that can be privately learned under the realizability assumption can be privately learned (i.e., in the general agnostic setting):

\begin{lem}[Special case of Theorem 4.16 in \cite{beimel2015learning}]\label{lem:equiv_priv_agnos_realiz}
Any class $\cH$ that is privately learnable with respect to all realizable distributions is also privately learnable (i.e., privately learnable in the general agnostic setting).
\end{lem}

% \begin{lem}[Special case of Theorem 4.16 in \cite{beimel2015learning}]\label{lem:equiv_priv_agnos_realiz}
% For all $\alpha, \beta >0, \eps <1, $If there is an $(\alpha, \beta, \eps, \delta)$-private learner with input sample size $n$ that learns a class $\cH$ with respect to realizable distributions, then there is an $(\alpha, \beta, \eps, \delta)$-private learner for $\cH$ with input sample size \[O\left(\frac{1}{\eps}\,\max\left(n, \frac{\VC(\cH)\,\log\left(\frac{1}{\alpha\,\beta}\right)}{\alpha^2}\right)\right).\]
% \end{lem}
\noindent The following fact follows from the private boosting technique due to \cite{DRV10}:

\begin{lem}[follows from Theorem~6.1 \cite{DRV10} (the full version)]\label{lem:priv-boost}
For any class $\cH$, under the realizability assumption, if there is a $\left(0.1, 0.1, 0.1\right)$-pure private learner for $\cH$, then $\cH$ is privately learnable by a pure private algorithm.
\end{lem}
We note that no analogous statement to the one in Lemma~\ref{lem:priv-boost} is known for approximate private learners. This is because it is not clear how one can scale down the $\delta$ parameter of the boosted learner to $\negl(n)$, as required by the definition of approximate DP learnability; specifically, the boosting result in \cite{DRV10} does not achieve this. On the other hand, the $\eps$  parameter of the boosted learner (according to \cite[Theorem~6.1]{DRV10}) can be scaled down by taking the input sample of the boosted learner to be large enough, and then apply the algorithm on a random subsample. The same technique would not be sufficient to scale down $\delta$ to $\negl(n)$. %{\color{blue}Added paragrpah above to discuss the lack of boosting for approx DP.}

\noindent We will also use the following notion of coverings:

\begin{defn}[$\alpha$-cover for a hypothesis class]\label{defn:cover}
A family of hypotheses $\tH$ is said to form an $\alpha$-cover for an hypothesis class $\cH\subseteq \{0, 1\}^{\cX}$ with respect to a distribution $\cD_{\cX}$ over $\cX$ if for every $h\in \cH$, there is $\tlh\in\tH$ such that $\dis\left(h, \tlh;~ \cD_{\cX}\right)\leq \alpha.$
\end{defn}

%% file: upper.tex
\section{Upper Bound}\label{sec:upper}

In this section we show that every VC class $\cH$ can be semi-privately learned in the agnostic case
with only $\tilde O(\VC(\cH)/\alpha)$ public examples:%{\color{blue} I am not sure if dropping VC makes sense? I don't think it's accurate to say every class is learnable with $1/\alpha$. Given the quantifiers, this sounds like no dependence on VC.}
\begin{thm}[Upper bound]\label{thm:upper-bound}
Let $\cH$ be a hypothesis class and let $\VC\left(\cH\right)=d$. For any $\alpha, \beta \in (0, 1), ~\eps>0,$ $\Aprv$ is an $(\alpha, \beta, \eps)$-semi-supervised semi-private agnostic learner for $\cH$ with private and public sample complexities:
\begin{align*}
\nprv&=O\left(\bigl(d\log(1/\alpha)+\log(1/\beta)\bigr)\,\max\left(\frac{1}{\alpha^2},~\frac{1}{\eps\,\alpha}\right)\right),\\
\npub&=O\left(\frac{d\log(1/\alpha)+\log(1/\beta)}{\alpha}\right).
\end{align*}
\end{thm}

\paragraph{Proof overview.} 
The upper bound is based on a reduction to the fact that any finite hypothesis class~$\mathcal{H}'$ can be learned privately with sample complexity (roughly) $O(\log\lvert \mathcal{H}'\rvert)$ via the exponential mechanism\footnote{The exponential mechanism is a basic algorithmic technique in DP \cite{mcsherry2007mechanism}.} \cite{KLNRS08}.
In more detail, we use the (unlabeled) public data to construct a finite class $\mathcal{H}'$ that forms a ``good enough {\it approximation}'' of the (possibly infinite) original class $\mathcal{H}$ (See description in Algorithm~\ref{Alg:exp-net}).
The relevant notion of approximation is captured by the definition of $\alpha$-cover (Definition~\ref{defn:cover}):
for every $h\in \mathcal{H}$ there exists~$h'\in\mathcal{H'}$ that $\alpha$-approximates $h$: $\dis\left(h, h';~ \cD_{\cX}\right)\leq \alpha$.
Indeed, it suffices to output an hypothesis~$h'\in \cH'$ that $\alpha$-approximates an optimal hypothesis $h^*\in\cH$. One interesting feature about this approach is that the constructed finite class (the $\alpha$-cover) depends on the distribution $\cD_{\cX}$ over $\cX$. Hence, the same constructed cover can be used to learn different target concepts as long as the distribution $\cD_{\cX}$ remains the same. 

Thus, the crux of the proof boils down to the question:
{\it How many samples from $\cD_\cX$ are needed in order to construct an $\alpha$-cover for $\cH$?}
It is not hard to see that (roughly) $O(\VC(\cH)/\alpha^2)$ examples suffice:
indeed, these many examples suffice to approximate the distances $\dis\left(h', h'';~ \cD_{\cX}\right)$ for every $h',h''\in\cH$, and therefore also suffice for constructing the desired $\alpha$-cover $\mathcal{H}'$.
We show how to reduce the number of examples to only
(roughly) $O(\VC(\cH)/\alpha)$ examples (Lemma~\ref{lem:cover}), which, by our lower bound, is nearly optimal

%\textcolor{red}{Shay: remove?}Our algorithm is essentially the same like the construction due to Beimel at al. \cite{beimel2013private}, where they consider a similar setting and analyze the realizable case. Our analysis draws on the connection to $\alpha$-coverings, and provides a simpler argument that extends to the agnostic setting. This argument leads to a nearly tight bound on private sample complexity in the agnostic setting. In fact, in the realizable setting, it is not hard to see that the same argument leads to nearly tight private sample complexity bound as well, which is (roughly) $O(\VC(\cH)/\alpha)$. However, we only consider the agnostic setting.  
%\textcolor{red}{Shay: remember to give credit to Stemmer et al.\ who used similar ideas.}{\color{blue}Raef: See if the way I cited them is fair and adequate.}

\begin{algorithm}
	\caption{$\Aprv$: Semi-Supervised Semi-Private Agnostic Learner}
	\begin{algorithmic}[1]
		\REQUIRE Private labeled dataset: $\Sprv=\{(x_1, y_1), \ldots, (x_{\nprv}, y_{\nprv})\}\in \mathcal{Z}^{\nprv}$, a public unlabeled dataset: $\Tpub=(\tx_1,\cdots, \tx_{\npub})\in \cX^{\npub}$, a hypothesis class $\cH\subset \{0, 1\}^{\cX}$, and a privacy parameter $\epsilon >0$.
		\STATE Let $\tT=\{\hx_1, \ldots,\hx_{\hatm}\}$ be the set of points $x\in\cX$ appearing at least once in $\Tpub$.
		\STATE Let $\Pi_{\cH}(\tT)=\left\{\left(h(\hx_1), \ldots, h(\hx_{\hatm})\right):~h\in\cH\right\}.$
		\STATE Initialize $\tH_{\Tpub}=\emptyset$.\label{step:init}
		\FOR{each $\bc=(c_1, \ldots, c_{\hatm})\in\Pi_{\cH}(\tT)$:}
		\STATE Add to $\tH_{\Tpub}$ arbitrary $h\in\cH$ that satisfies $h(\hx_j)=c_j$ for every $j=1,\ldots, \hatm$.  \label{step:rep-hyp}
		\ENDFOR
		\STATE Use the exponential mechanism with inputs $\Sprv,~\tH_{\Tpub}, \eps$ and score function $q(\Sprv, h)\triangleq -\herr(h; \Sprv)$ to select $\hpv\in \tH_{\Tpub}$. \label{step:exp-mech} 
		\RETURN $\hpv.$
		
	\end{algorithmic}
	\label{Alg:exp-net}
\end{algorithm}

%\textcolor{red}{Theorem 3.1 is redundant, right? Perhaps you meant to state it as a lemma after Theorem 3.2?}

%\begin{thm}[Upper bound on private/ public sample complexity]\label{thm:upper-bound}
%Let $\cH$ be a hypothesis class with $\VC(\cH)=d$. Let $\alpha, \beta \in (0, 1)$. Let $m=O\left(\frac{d\log(1/\alpha)+\log(1/\beta)}{\alpha}\right)$ and $n=O\left(\frac{d\log(1/\alpha)+\log(1/\beta)}{\alpha^2}\right)$. Let $\cD$ be any distribution over $\cX\times \cY$ and let $\cD_{\cX}$ denote the marginal distribution over $\cX$. Given input labeled sample $S\sim\cD^n$ and unlabeled sample $T\sim\cD_{\cX}^m$, then with probability at least $1-\beta,$ Algorithm~\ref{Alg:exp-net} outputs $\hpv\in\cH$ satisfying 
%$\err(\hpv; \cD)\leq \alpha+ \gamma$, where $\gamma\triangleq \min\limits_{h\in\cH}\err(h; \cD).$
%\end{thm}

%\textcolor{red}{The private sample complexity should depend on $\eps$, right?}

The proof of Theorem~\ref{thm:upper-bound} relies on the following lemmas. 

\begin{lem}\label{lem:priv}
For all $\Tpub\in \cX^{\npub},$ $\Aprv(\cdot, ~\Tpub)$ is $\epsilon$-differentially private (with respect to its first input $\Sprv$). 
\end{lem}

\begin{proof}
%The proof of privacy follows the similar line of argument as in \cite{beimel2013private}{\color{blue}added the previous line}. 
For any fixed $\Tpub\in \cX^{\npub}$, note that the hypothesis class $\tH_{\Tpub}$ constructed in Step~\ref{step:rep-hyp} depends only on the public dataset.  Note also that the private dataset is used only in Step~\ref{step:exp-mech}, which is an instantiation of the generic learner of \cite{KLNRS08}. The proof thus follows directly from \cite[Lemma~3.3]{KLNRS08}.
\end{proof}

\begin{lem}[$\alpha$-cover for $\cH$]\label{lem:cover}
 Let $\Tpub\sim\cD^{\npub}_{\cX},$ where $\npub=O\left(\frac{d\,\log(1/\alpha)+\log(1/\beta)}{\alpha}\right)$. Then, with probability at least $1-\beta,$ the family $\tH_{\Tpub}$ constructed in Step~\ref{step:rep-hyp} of Algorithm~\ref{Alg:exp-net} is an $\alpha$-cover for $\cH$ w.r.t. $\cD_{\cX}.$ 
\end{lem}

\begin{proof}
We need to show that with high probability, for every $h\in\cH$ there exists $\tlh\in\tH_{\Tpub}$ such that $\dis(h,\tlh;\cD_{\cX})\leq \alpha$.
Let~$\tT=\{\hx_1, \ldots, \hx_{\hatm}\}$ be the set of points in $\cX$ that appears at least once in $\Tpub$,
and let $h(\tT)=\left(h(\hx_1), \ldots, h(\hx_{\hatm})\right)$. By construction, there must exist $\tlh\in\tH_{\Tpub}$ such that 
\[\tlh(\hx_j)=h(\hx_j) \quad~ \forall j\in [\hatm]\]
that is, $\hdis\left(\tlh, h; ~\Tpub\right)=0$; we will show that $\dis(h,\tlh;\cD_{\cX})\leq \alpha$. %with high probability over the choice of $\Tpub$.
For $\Tpub\sim \cD_{\cX}^{\npub},$ define the event 
\begin{align}
\bad&=\left\{\exists h_1, h_2 \in \cH:~ \dis\left(h_1, h_2; \cD_{\cX}\right)>\alpha ~\text{ and } ~\hdis\left(h_1, h_2; ~\Tpub\right)=0\right\}\nonumber  
\end{align}
We will show that 
\begin{align}
    \pr{\Tpub\sim\cD_{\cX}^{\npub}}{\bad}&\leq 2\left(\frac{2e\,\npub}{d}\right)^{2d}\,e^{-\alpha\,\npub/4}.\label{ineq:pr_bad_event}
\end{align}

Before we do so, we first show that (\ref{ineq:pr_bad_event}) suffices to prove the lemma. Indeed, if 
$\dis\left(\tlh, h; ~\cD_{\cX}\right)> \alpha$ for some $h\in\cH$ then the event $\bad$ occurs; in other words if $\bad$ does not occur then
$\tH_{\Tpub}$ is an $\alpha$-cover. Hence, 
\[\pr{\Tpub\sim\cD_{\cX}^{\npub}}{\text{$\tH_{\Tpub}$ is not an $\alpha$-cover}}\leq 2\left(\frac{2e\,\npub}{d}\right)^{2d}\,e^{-\alpha\,\npub/4}.\] 
Now, via standard manipulation, this bound is at most $\beta$ when $\npub=O\left(\frac{d\,\log(1/\alpha)+\log(1/\beta)}{\alpha}\right)$, which yields the desired bound and finishes the proof. 

Now, it is left to prove (\ref{ineq:pr_bad_event}). 
To do so, we use a standard VC-based uniform convergence bound (a.k.a $\alpha$-net bound) on the class $\cHd\triangleq\left\{h_1\Delta h_2: h_1, h_2\in\cH\right\}$ where $h_1\Delta h_2:\cX\to\{0,1\}$ is defined as 
\[ h_1\Delta h_2(x)\triangleq \ind\left(h_1(x)\neq h_2(x)\right) \quad~ \forall x\in\cX\] 
%Note that the \emph{all-zero} hypothesis, defined as $h_{\zeroB}(x)=0~ \forall x\in\cX$, is in $\cHd$. 
Let $\cG_{\cHd}$ denote the growth function of $\cHd$; that is, for any number $m,$ 
\[\cG_{\cHd}(m)\triangleq \max\limits_{V: \lvert V\rvert=m}\lvert\Pi_{\cHd}(V)\rvert,\] 
where $\Pi_{\cHd}(V)$ is the set of all possible dichotomies that can be generated by $\cHd$ on a set $V$ of size~$m$. Note that $\cG_{\cHd}(m)\leq \left(\frac{e\,m}{d}\right)^{2d}.$ This follows from the fact that for any set $V$ of size $m$, we have $\lvert\Pi_{\cHd}(V)\rvert\leq \lvert\Pi_{\cH}(V)\rvert^2$ since every dichotomy in $\Pi_{\cHd}$ is determined by a pair of dichotomies in $\Pi_{\cH}(V)$. Hence,  $\cG_{\cHd}(m)\leq \left(\cG_{\cH}(m)\right)^2\leq \left(\frac{e\,m}{d}\right)^{2d},$ where the last inequality follows from Sauer's Lemma~\cite{sauer1972density}. %{\color{blue} add citation}
Now, by invoking a uniform convergence argument, we have
\begin{align*}
&\pr{\Tpub\sim\cD_{\cX}^{\npub}}{\exists h_1, h_2 \in \cH:~ \dis\left(h_1, h_2;~\cD_{\cX}\right)>\alpha ~ \text{ and }~ \hdis\left(h_1, h_2; ~\Tpub\right)=0}\\
=&\pr{\Tpub\sim\cD_{\cX}^{\npub}}{\exists h \in \cHd:~ \dis\left(h, h_{\zeroB};~\cD_{\cX}\right)>\alpha ~ \text{ and }~ \hdis\left(h, h_{\zeroB}; ~\Tpub\right)=0}\\
\leq& ~2 \cG_{\cHd}(2\,\npub)\,e^{-\alpha\,\npub/4}\\
\leq& ~2\left(\frac{2e\,\npub}{d}\right)^{2d}\,e^{-\alpha\,\npub/4}.
\end{align*}
%\textcolor{red}{Shay: perhaps also add that in the last inequality we use the standard fact that ${n \choose \leq k} \leq (en/k)^k$.}

The bound in the third line is non-trivial; it follows from the so-called double-sample argument which was used by Vapnik and Chervonenkis in their seminal paper~\cite{vapnik2015uniform}.
The same argument is used in virtually all VC-based uniform convergence bounds (see, e.g.,~\cite[Sec. 28.3]{shalev2014understanding}). 

This proves inequality~(\ref{ineq:pr_bad_event}) and completes the proof of the lemma.

%{\color{blue} @Shay: I think it's quite standard, I added a reference right before the sequence of inequalities.}

%{\color{red}
%Right now it reads as if this proof is self contained.
%I think we should at least notify the reader that there is a leap here and that for more details they should see the reference (alternatively, we could just remove this calculation altogether and apply the epsilon-net theorem directly on $\cHd$ (the version of it which depends on the growth function).}

\end{proof}

\subsection*{Proof of the Upper Bound (Theorem~\ref{thm:upper-bound})}
First, we note that $\eps$-differential privacy of $\Aprv$ follows from Lemma~\ref{lem:priv}. Thus, it is left to establish the accuracy guarantee of $\Aprv$ and the sample complexity bounds on~$\npub$ and~$\nprv$.
Let 
\[h^*\in \arg\min\limits_{h\in \cH}\err\left(h; ~\cD\right)\] 
denote the optimal hypothesis in $\cH$.
We will show that with probability $\geq 1-\beta$, the output hypothesis $\hpv$ satisfies
$\err\left(\hpv;~ \cD\right)\leq \err\left(h^*;~\cD\right) +\alpha.$

First, fix the randomness in the choice of $\Tpub$. Let $\tH_{\Tpub}$ denote the corresponding realization of the finite class generated in Steps~\ref{step:init}-\ref{step:rep-hyp} of Algorithm~\ref{Alg:exp-net}. 
Let
\[h^*_{\Tpub}\triangleq\arg\min\limits_{h\in\tH_{\Tpub}}\err(h;~\cD)\] 
denote the optimal hypothesis in $\tH_{\Tpub}$.
Using the result in \cite[Theorem~3.4]{KLNRS08} for the generic learner based on the exponential mechanism, it follows that a private sample size $$\nprv=O\left(\left(\log\left(\lvert\tH_{\Tpub}\rvert\right)+\log(1/\beta)\right)\max\left(\frac{1}{\alpha^2},~\frac{1}{\eps\,\alpha}\right)\right)$$ 
suffices to ensure that, w.p. $\geq 1-\beta/2$ (over randomness in $\Sprv$ and in the exponential mechanism), we have $\err\left(\hpv;~ \cD\right)\leq \err\left(\hat{h}_{\Tpub};~\cD\right)+ \alpha/2.$ 
%\textcolor{red}{Shouldn't there be a dependence on $\eps$?}{\color{blue} Added it in the sample bound $\nprv$}
From the setting of $\npub$ in the theorem statement together with Sauer's Lemma, it follow that 
\begin{align*}
\log\left(\lvert\tH_{\Tpub}\rvert\right)\leq d\,\log(\frac{e\, \npub}{d})
&\leq
O\bigl(d\log(\frac{d\log(1/\alpha) + \log(1/\beta)}{d\alpha})\bigr)\\
&=O\left(d\left(\log(1/\alpha)+\log\left(\log(1/\alpha)+\frac{\log(1/\beta)}{d}\right)\right)\right)\\
&=O\left(d\log(1/\alpha)+d\left(\log\left(\frac{\log(1/\beta)}{d}\right)\right)^{+}\right)
%&= \textcolor{red}{O\bigl(d\log(1/\alpha)+ d\log\log(1/\beta)\bigr).}
%O\left(d\log\left(1/\alpha\right)+\log\log\left(1/\beta\right)\right).
\end{align*}
where $(x)^{+}\triangleq \max(0, x)$.

Hence,
\begin{align*}
\nprv&=O\left(\left(d\log(1/\alpha)+ d\left(\log\left(\frac{\log(1/\beta)}{d}\right)\right)^{+} +\log(1/\beta)\right)\,\max\left(\frac{1}{\alpha^2},~\frac{1}{\eps\,\alpha}\right)\right)\\
&=O\left(\left(d\log\left(1/\alpha\right)+\log\left(1/\beta\right)\right)\,\max\left(\frac{1}{\alpha^2},~\frac{1}{\eps\,\alpha}\right)\right)
\end{align*}

This yields the bound on $\nprv$ as in the theorem statement.
Now, by invoking Lemma~\ref{lem:cover}, it follows that for the setting of $\npub$ as in the theorem statement, w.p. $\geq 1-\beta/2$ over the randomness in $\Tpub$, we have $\dis\left(\hat{h}_{\Tpub}, h^*; ~\cD_{\cX}\right)\leq \alpha/2.$ Hence, by the triangle inequality,  
$\err\left(\hat{h}_{\Tpub};~\cD\right)-\err\left(h^*;~\cD\right)\leq \dis\left(\hat{h}_{\Tpub}, h^*; ~\cD_{\cX}\right)\leq \alpha/2.$
This completes the proof of the theorem.

%% file: lower.tex
\section{Lower Bound}

%\rnote{introduce the lower bound}

In this section we establish that the upper bound on the public sample complexity which was derived in the previous section is nearly tight. 

\begin{thm}[Lower bound for classes of infinite Littlestone dimension]\label{thm:lower-bound}
Let $\cH$ be any class with an infinite Littlestone dimension (e.g., the class of thresholds over $\mathbb{R}$). 
Then, any semi-private learner for $\cH$ must have public sample of size $~\npub = \Omega(1/\alpha)$, where $\alpha$ is the excess error.
\end{thm}

In the case of pure differentially privacy we get a stronger statement which manifests a dichotomy
that applies for every class:
\begin{thm}[Pure private vs. pure semi-private learners]\label{thm:dichotomy}
Every class $\cH$ must satisfy exactly one of the following: 
\begin{enumerate}
    \item $\cH$ is learnable by a pure private learner. 
    \item Any pure semi-private learner for $\cH$ must have public sample of size $~\npub = \Omega(1/\alpha)$, where $\alpha$ is the excess error
\end{enumerate}
\end{thm}

\paragraph{Proof overview.}
The crux of the argument is a {\it public-data-reduction lemma} (Lemma~\ref{lem:reduction}), 
which shows how one can reduce the number of public examples 
at the price of a proportional increase in the excess error.
This lemma implies, for example, that if $\cH$ can be learned up to an excess error of $\alpha$ 
with less than $\frac{1}{1000\alpha}$ public examples
then it can also be privately learned without any public examples and excess error of at most $<\frac{1}{10}$.
Stating contra-positively, if $\cH$ can not be privately learned with excess error $<\frac{1}{10}$
then it can not be semi-privately learned up to an excess error of $\alpha$ 
with less than $\frac{1}{1000\alpha}$ public examples.
This yields a lower bound of $\Omega(1/\alpha)$ on the public sample complexity
for every class $\cH$ which is not privately learnable with constant excess error

One example for such a class is any class with infinite Littlestone dimension (e.g., the class of $1$-dimensional thresholds over an infinite domain). This follows from the result in \cite{almm19}:

\begin{thm}[Restatement of Corollary~2 in \cite{almm19}]\label{thm:imposs_thres}
Let $\cH$ be any class of infinite Littlestone dimension (e.g., the class of thresholds over an infinite domain $\cX\subseteq\re$). For any $n\in \mathbb{N},$ given a private sample of size $n$, there is no $\left(\frac{1}{16},~\frac{1}{16},~0.1,~ \frac{1}{100\,n^2\log(n)}\right)$-private learner for $\cH$ (even in the realizable case).
\end{thm}

% \begin{lem}[Restatement of Corollary~2 in \cite{almm19}]\label{lem:imposs_thres}
% The class of thresholds over an infinite domain $\cX\subseteq\re$ is not privately learnable.
% \end{lem}

A special case of the above result was first proven in \cite{bun2015differentially}, where it was shown that no \emph{proper} private learner can learn thresholds over an infinite domain. A proper learner is bound to output a hypothesis from the given class. Our definitions in this paper for private and semi-private learners do not make this restriction on the learner; that is, the learners in those definitions can be \emph{non-proper}, i.e., they are allowed to output a binary hypothesis that is not necessarily in the given class $\cH$.

\begin{remark}
The aforementioned reduction we use for the lower bound holds even when the public sample is \emph{labeled}. This makes the lower bound stronger since it holds even in the fully supervised setting of semi-private learning described in Definition~\ref{defn:pp-learner}. We also note that this reduction holds for both \emph{pure} and \emph{approximate} private/semi-private learners.
\end{remark}

%We also note that since Lemma~\ref{lem:imposs_thres} holds for \emph{possibly non-proper} private learners, our lower bound also applies to non-proper learners. {\color{blue}Raef: Rephrased the remark and the preceding paragraph and added a citation to Bun et al.}

% {\color{blue} Raef: Also, emphasized that the lower bound applies to non-proper learners.}
% \textcolor{red}{Shay: if we stress the issue of proper vs.\ non-proper then we might want to cite the lower bound for proper learnability of thresholds given by Bun et al.}
% \textcolor{red}{Shay: Perhaps instead we can just note more generally that the above public-data-reduction lemma preserves properties of the learner such proper/non-proper and agnostic/realizable 
% and as such can be used in more specialized lower bounds which focus on such learners. }

We now formally state and prove the reduction outlined above. 

%\rnote{new lemma and proof}
\begin{lem}[Public data reduction lemma]\label{lem:reduction}
Let $0<\alpha\leq 1/100,~ \eps >0, ~ \delta\geq 0$. Suppose there is an $(\alpha, \frac{1}{18}, \eps, \delta)$-agnostic semi-private learner for a hypothesis class $\cH$ with private sample size $\nprv$ and public sample size $\npub$. Then, there is a $\left(100\,\npub\,\alpha, ~\frac{1}{16}, \eps, \delta\right)$-private learner that learns any distribution realizable by $\cH$ with input sample size $\lceil\frac{\nprv}{10\,\npub}\rceil$.
\end{lem}

%\textcolor{red}{Shay: I think we can probably strengthen this lemma and show that the private sample complexity
%of the implied fully private learner is $O(\nprv/\npub)$ (because the way we construct the mixed distribution, only one out of $\npub$ samples comes from the target distribution and we can ignore the rest of the samples).} {\color{blue} Raef: I thought about this, but I thought this maybe a complication that is not necessary for the main result. Let's decide about this later.}

%\begin{algorithm}
%\caption{A Private Learner}
%\label{alg:priv}

%{\bf Input:} 
%a realizable private sample $S$ of length $n=\nprv/(10\cdot \npub)$.

%Run $\mathcal{B}$ on the following public/private $S_{pub},S_{priv}$

%\end{algorithm}

\begin{proof}
%Let $\Z=\cX\times\{0, 1\}$ denote the (labeled) examples domain. Let $\cH$ be any hypothesis class over $\cX$. Let $\cA$ be $(\alpha, \frac{1}{12}, \eps, \delta)$-agnostic semi-private learner for $\cH$ with private and public sample sizes $\nprv$ and $\npub$, respectively. Using $\cA$, we will construct a $\left(100\,\npub\,\alpha, ~\frac{1}{10}, \eps, \delta\right)$-private learner $\cB$ that learns any distributions realizable by $\cH$ with input sample size $\frac{\nprv}{10\,\npub}$.  We now describe $\cB$.

Let $\cA$ denote the assumed agnostic-case semi-private learner for $\cH$ with input private sample of size $\nprv$ and input public sample of size $\npub$. Using $\cA$, we construct
a realizable-case private learner for $\cH$, which we denote by $\cB$. The description of $\cB$ appears in Algorithm~\ref{alg:priv}. 

\begin{algorithm}
\caption{Description of the private learner $\cB$:} 
\label{alg:priv}
\begin{algorithmic}[1]
\REQUIRE Private sample $\tS=(\tz_1, \ldots, \tz_{\tn})$ of size $\tn=\lceil\nprv/(10\cdot\npub)\rceil$.
\STATE Pick a fixed (dummy) distribution $\cD_0$ over $\cZ=\cX\times\{0,1\}$ where the label $y\in\{0,1\}$ is drawn uniformly at random from $\{0, 1\}$ independently from $x\in\cX$.
\STATE Set $p={1}/(100\cdot \npub)$.
\STATE Using $\tS$ and $\cD_0$, construct samples $\Sprv,~\Spub$ using procedures $\prvs(\tS, \cD_0, p, \nprv)$ and $\pubs(\tS, \cD_0, \npub)$ given by Algorithms~\ref{alg:privsamp} and \ref{alg:pubsamp} below. 
%\STATE Run $\cA$ on input private and public samples $\Sprv$ and $\Spub$
\STATE Return $\tlh=\cA(\Sprv, \Spub)$.
%\STATE \textcolor{red}{Shay: Here I want to put the description of $\Sprv,\Spub$, perhaps even add a small break 
%(the descriptions of $\Sprv,\Spub$ are below)}
\end{algorithmic}
\end{algorithm}

%Choose a fixed (dummy) distribution $\cD_0$ over $\Z$, where the label is drawn uniformly at random from $\{0, 1\}$. Let $\tS=(\tz_1, \ldots, \tz_{\tn})$ be the input private sample of $\cB$ whose size is $\tn=\nprv/(10\cdot\npub)$. Let $p=\frac{1}{100\,\npub}$. Construct two samples $\Sprv=(\zprv_1, \ldots, \zprv_{\nprv})$ and $\Spub=(\zpub_1, \ldots, \zpub_{\npub})$ as follows: 

\begin{algorithm}
\caption{Private Sample Generator $\prvs$:} 
\label{alg:privsamp}
\begin{algorithmic}[1]
\REQUIRE Sample $\tS=(\tz_1, \ldots, \tz_{\tn})$, Distribution $\cD_0,$ parameter $p$, sample size $\nprv$.
\STATE $i:=1$
\WHILE{$\tS\neq \emptyset$ \textbf{and} $i\leq \nprv$:}
\STATE Sample $b_i\sim\Ber(p)$ (independently for each $i$), where $\Ber(p)$ is Bernoulli distribution with mean $p$.
\IF{$b_i=1$:}
\STATE Set $\zprv_i$ to be the next element in $\tS$, i.e., $\zprv_i=\tz_{j_i},$ where $j_i=\sum_{k=1}^i b_k$.
\STATE Remove this element from $\tS$: $\tS\leftarrow \tS\setminus \tz_{j_i}$.
\ELSE 
\STATE Set $\zprv_i=\zz_i,$ where $\zz_i$ is a fresh independent example from the ``dummy'' distribution $\cD_0$.
\ENDIF
\STATE $i\leftarrow i+1$
\ENDWHILE
\RETURN $\Sprv=(\zprv_1, \ldots, \zprv_{\nprv}).$
\end{algorithmic}
\end{algorithm}

\begin{algorithm}
\caption{Public Sample Generator $\pubs$:} 
\label{alg:pubsamp}
\begin{algorithmic}[1]
\REQUIRE Sample $\tS=(\tz_1, \ldots, \tz_{\tn})$, Distribution $\cD_0,$ sample size $\npub$.
\FOR{$i=1, \ldots, \npub:$}
\STATE Set $\zpub_i = \zz_i$ where $\zz_i$ is a fresh independent example from $\cD_0$.
\ENDFOR
\RETURN $\Spub=(\zpub_1, \ldots, \zpub_{\npub})$ 
 \end{algorithmic}
\end{algorithm}

%Then, $\cB$ runs $\cA$ on input private and public samples $\Sprv$ and $\Spub$, respectively, and returns the output hypotheses $\tlh=\cA(\Sprv, \Spub).$

The following two claims about $\cB$ establish its privacy and accuracy guarantees.

\begin{claim}[Privacy guarantee of $\cB$]\label{c:privacy}
$\cB$ is $(\eps, \delta)$-differentially private
\end{claim}
This follows directly from the fact that for any realization of $\Spub,$ $\cA(\cdot, \Spub)$ is $(\eps, \delta)$-differentially private, the fact that $\Spub$ does not contain any points from $\tS$, and the fact that each point in $\tS$ appears at most once in $\Sprv$. 

Thus, it remains to show that
\begin{claim}[Accuracy guarantee of $\cB$]\label{c:utility}
Let $\cD$ be any distribution over $\Z$ that is realizable by $\cH$. Suppose $\tS\sim\cD^{\tn}$. Then, except with probability at most $1/16$ (over the choice of $\tS$ and internal randomness in $\cB$), the output hypothesis $\tlh$ satisfies: $\err(\tlh; ~\cD)\leq 100\,\npub\,\alpha.$
\end{claim}
Let $\cDp$ denote the mixture distribution $p\cdot\cD+(1-p)\cdot\cD_0$ (recall the definition of $p$ from Algorithm~\ref{alg:priv}). 
To prove Claim~\ref{c:utility}, we first show that both $\Sprv$ and $\Spub$ can be viewed as being sampled from $\cDp$ with almost no impact on the analysis. Then, using the fact that $\cA$ learns $\cH$ with respect to $\cDp$, the claim will follow.

First, note that since $\tn=10\cdot p\cdot\nprv,$ then by Chernoff's bound, except with probability $<0.01,$ Algorithm~\ref{alg:privsamp} exits the \textbf{WHILE} loop with $i=\nprv$. Thus, except with probability~$<0.01,$ we have  
\begin{align}
\lvert\Sprv\rvert&=\nprv, \text{  hence, } \Sprv\sim\cDp^{\nprv}. \label{eq:bad_event1}  
\end{align}
%where $\Sprv$ is the sample generated by $\cB$ as the input private sample of $\cA$. 
As for $\Spub$, note that $\Spub=(\zz_1, \ldots, \zz_{\npub})\sim\cD_0^{\npub}$, 
and therefore we can not use the same argument we used with $\Sprv$. 
Instead, we will show that $\cD_0^{\npub}$ is close in total variation to $\cDp^{\npub}$. Let $\hSpub=(\hz_1, \ldots, \hz_{\npub})$ be i.i.d.\ sequence generated as follows: for each $i\in [\npub],$ $\hz_i=b_i\,v_i+(1-b_i)\,\zz_i,$ where $(b_1, \ldots, b_{\npub})\sim \left(\Ber(p)\right)^{\npub}$, and $(v_1, \ldots, v_n)\sim\cD^{\npub}$. It is clear that $\hSpub\sim\cDp^{\npub}$. Moreover, observe that 
\begin{align*}
\pr{}{\hSpub=\Spub}&\geq\pr{}{b_i=0~~\forall ~i\in [\npub]}=\left(1-\frac{1}{100\,\npub}\right)^{\npub}\geq 0.99    
\end{align*}
%This implies that the total variation between $\hSpub$ and $\Spub$ is at most $0.01$. 
Note that $\pr{}{\hSpub\neq \Spub}$ is the probability measure attributed to the first component of the mixture distribution $\cDp$ of $\hat{S}_{pub}$  (i.e., the component from $\cD$). Hence, it follows that the total variation between the distribution of $\hat{S}_{pub}$ (induced by the mixture $\cDp$) and the distribution of $S_{pub}$ (induced by $\cD_0$) is at most $0.01$. In particular, the probability of any event w.r.t. the distribution of $\hSpub$ is at most $0.01$ far from the probability of the same event w.r.t. the distribution of $\Spub$. Hence,
\begin{align}
    &\pr{\Sprv, \Spub, \cA}{\err\left(\cA(\Sprv, \Spub);~\cDp\right)-\min\limits_{h\in\cH}\err(h; \cDp) >\alpha} \nonumber \\
    &-\pr{\Sprv, \hSpub, \cA}{\err\left(\cA(\Sprv, \Spub);~\cDp\right)-\min\limits_{h\in\cH}\err(h; \cDp) >\alpha}\leq 0.01\label{ineq:bad_event2}
\end{align}
Now, from (\ref{eq:bad_event1}) and the premise that $\cA$ is agnostic semi-private learner, we have 
\begin{align*}
    \pr{\Sprv, \hSpub, \cA}{\err\left(\cA(\Sprv, \Spub);~\cDp\right)-\min\limits_{h\in\cH}\err(h; \cDp) >\alpha}&\leq \frac{1}{17}
\end{align*}
Hence, using (\ref{ineq:bad_event2}), we conclude that except with probability $<1/16$,
\begin{align}
    \err\left(\cA(\Sprv, \Spub);~\cDp\right)-\min\limits_{h\in\cH}\err(h; \cDp) &\leq \alpha.\label{ineq:err_semi-priv}
\end{align}
Note that for any hypothesis $h$, 
$$\err(h;~ \cDp)=p\cdot\err(h; ~\cD)+(1-p)\cdot\err(h; \cD_0)=p\cdot\err(h; ~\cD)+\frac{1}{2}(1-p),$$ 
where the last equality follows from the fact that the labels generated by $\cD_0$ are completely noisy (uniformly random labels). Hence, we have $\arg\min\limits_{h\in\cH}\err(h; \cDp)=\arg\min\limits_{h\in\cH}\err(h; \cD)$. That is, the optimal hypothesis with respect to the realizable distribution $\cD$ is also optimal with respect to the mixture distribution~$\cDp$. Let $h^*\in\cH$ denote such hypothesis. Note that $\err(h^*; \cD)=0$ and $\err(h^*; \cDp)=\frac{1}{2}(1-p)$. These observations together with (\ref{ineq:err_semi-priv}) imply that except with probability $<1/16$, we have 
\begin{align*}
    \alpha&\geq p\cdot \err\left(\cA(\Sprv, \Spub);~\cD\right)
\end{align*}
Hence, $\err\left(\cB(\tS); \cD\right)=\err\left(\cA(\Sprv, \Spub);~\cD\right)\leq 100\cdot\npub\cdot\alpha.$
This completes the proof.

% That is, $\err\left(\cA(\Sprv, \Spub);~\cD\right)\leq 100\cdot\npub\cdot\alpha.$
% This completes the proof.

%\textcolor{red}{Shay:
%I would replace the above consideration by the following consideration:
%assuming that the ``While'' loop finishes with $i=\nprv$ (which happens with a large probability),
%we have that the input of $\cB$ is distributed like inputs from $\cDp$ conditioned on that
%all the public examples are dummy. Let $E$ denote this event, 
%let $h, h_E, h_{\bar E}$ denote the output hypotheses of $\cB$ when the input is drawn from $\cDp$,
%when the input is drawn from $\cDp$ conditioned on $E$, and when the input is drawn from $\cDp$ conditioned on the complement of $E$ (respectively). Thus, $h_E$ is the output of our algorithm. Note that
%\[
%\err(h ; \cDp) = \Pr[E]\err(h_E ; \cDp) + (1-\Pr[E])\err(h_{\bar E} ; \cDp),
%\]
%which implies that
%\[
%\err(h_E ; \cDp) \leq \frac{1}{\Pr[E]}\err(h ; \cDp)\leq 2\err(h ; \cDp) 
%\]
%}

\end{proof}

With Lemma~\ref{lem:reduction}, we are now ready to prove the main results for this section:
%\rnote{theorem and proof recently modified}
% \begin{thm}[Lower bound for classes of infinite Littlestone dimension]\label{thm:lower-bound}
% Let $\cH$ be any class of infinite Littlestone dimension (e.g., the class of thresholds over an infinite domain). For any $\alpha > 0,$ suppose there is an $\left(\alpha, \frac{1}{18}, \eps, \delta\right)$-semi-private learner for $\cH$ with private sample size $\nprv$, $\eps=0.1$, and $\delta=\frac{1}{100\, \nprv^2\log(\nprv)}$. Then, the public sample size must satisfy $\npub=\Omega\left(\frac{1}{\alpha}\right).$
% \end{thm}

\subsubsection*{Proof of Theorem~\ref{thm:lower-bound}}
\begin{proof}
Suppose $\cA$ is a semi-private learner for~$\cH$ with sample complexities $\nprv,\npub$. 
In particular, given $\nprv(\alpha,\frac{1}{18}),\npub(\alpha,\frac{1}{18})$ private and public examples, 
$\cA$ is $(\alpha, ~\frac{1}{18}, ~0.1,~ \frac{1}{100\, \nprv^2\log(\nprv)})$-semi-private learner for~$\cH$. Hence, by Lemma~\ref{lem:reduction}, there is $(100\npub\alpha, ~\frac{1}{16}, ~0.1,~ \frac{1}{100\, \nprv^2\log(\nprv)})$-private learner for~$\cH$. Thus, Theorem~\ref{thm:imposs_thres} implies that $100\npub\alpha > \frac{1}{16}$ and hence that $\npub > \frac{1}{1600\,\alpha}$ as required.
\end{proof}

\subsubsection*{Proof of Theorem~\ref{thm:dichotomy}}
\begin{proof}
First, if $\cH$ is learnable by a pure private learner, then trivially the second condition cannot hold since $\cH$ can be learned without any public examples. Now, suppose that the first item does \emph{not} hold. Note that by Lemma~\ref{lem:equiv_priv_agnos_realiz}, this implies that there is \emph{no} pure private learner for $\cH$ with respect to realizable distributions. By Lemma~\ref{lem:priv-boost}, this in turn implies that there is \emph{no} $\left(\frac{1}{16}, ~\frac{1}{16}, ~0.1\right)$-pure private learner for $\cH$ with respect to realizable distributions. Now, suppose $\cA$ is a pure semi-private learner $\cA$ for $\cH$. Then, this implies that for any $\alpha>0$, $\cA$ is an $\left(\alpha, ~\frac{1}{18}, ~0.1\right)$-pure semi-private learner for $\cH$ {with sample complexities $\nprv(\alpha,\frac{1}{18}),\npub(\alpha,\frac{1}{18})$}. Hence, by Lemma~\ref{lem:reduction}, there is a $\left(100\,\npub\,\alpha, ~\frac{1}{16}, ~0.1\right)$-pure private learner for $\cH$ w.r.t. realizable distributions. This together with the earlier conclusion implies that $100\,\npub\,\alpha> \frac{1}{16},$ and therefore that $\npub> \frac{1}{1600\,\alpha},$ which shows that the condition in the second item holds. 
\end{proof}